\documentclass[11pt]{article}
\usepackage[margin=1in]{geometry}
\usepackage{amsmath}
\usepackage{amssymb}
\usepackage{amsthm}
\usepackage{algorithm}
\usepackage{algpseudocode}
\usepackage{booktabs}
\usepackage{graphicx}
\usepackage{xcolor}
\usepackage{hyperref}
\hypersetup{hidelinks}
\usepackage{url}
\usepackage{microtype}
\usepackage[round,authoryear]{natbib}

\newtheorem{proposition}{Proposition}

\title{WDL-OPD: Weak-Driven On-Policy Distillation via Mixture-Constrained Co-Training}
\author{%
  Zehao Chen$^{1,2,*}$ \quad
  Gongxun Li$^{1,2,*}$ \quad
  Tianxiang Ai$^{2,*}$ \quad
  Yifei Li$^{1,2}$ \\
  Zixuan Huang$^{1}$ \quad
  Wang Zhou$^{2}$ \quad
  Tao Huang$^{2}$ \quad
  Fuzhen Zhuang$^{1}$ \\
  Xianglong Liu$^{1}$ \quad
  Jianxin Li$^{1}$ \quad
  Deqing Wang$^{1,\dagger}$ \quad
  Yikun Ban$^{1,\dagger}$ \\
  $^{1}$Beihang University \quad
  $^{2}$China Telecom eSurfing Cloud \\
  $^{*}$Equal contribution \quad
  $^{\dagger}$Corresponding authors \\
  {\small
    \texttt{zehaochenacid@buaa.edu.cn}}
}
\date{August 2026}

\begin{document}

\maketitle

\begin{abstract}
On-policy distillation (OPD) aligns a student with a teacher on trajectories
sampled from the student itself, reducing the train--test state mismatch of
offline distillation.  The same feedback loop can nevertheless be unstable:
each update changes both the policy and the states on which the next update is
computed.  We introduce \textbf{WDL-OPD}, a mixture-constrained co-training
method with two trainable policies.  An anchor policy generates every rollout,
an auxiliary policy evaluates the same visited states, and a geometric mixture
of their token distributions is matched to a frozen teacher by reverse KL.
Both policies receive gradient.  We show that freezing the auxiliary recovers
an anchor-plus-contrast proxy target closely related to OPD$^2$ and W2S-OPD,
whereas joint training creates branch-level degrees of freedom that a static
delta cannot express.  In recorded Qwen3 experiments at 1.7B and 4B scale,
WDL-OPD produces the strongest student checkpoint in each of four
scale--domain settings.  It raises MATH500 accuracy from $0.630$ to $0.685$ at
4B and from $0.521$ to $0.585$ at 1.7B.  In code generation, seven
single-policy OPD configurations exhibit entropy growth or trajectory
degradation, while co-training reaches independently re-evaluated development
scores of $0.637$ and $0.375$.  Because several comparisons differ in
curriculum or initialization, these results support a stabilization hypothesis
rather than a universal causal claim.  We provide the exact training algorithm,
failure evidence, and the controlled comparison matrix needed to test that
hypothesis.
\end{abstract}

\section{Introduction}

On-policy distillation (OPD) trains a student on its own generated states while
a teacher supplies dense token-level supervision
\citep{agarwal2024opd,thinkingopd}.  Compared with distillation on fixed
teacher trajectories, this removes a central source of exposure bias: the
teacher is queried at prefixes that the deployed student can actually visit.
OPD has consequently become a practical alternative or complement to
reward-based post-training.

The benefit comes with a dynamical cost.  An OPD update changes the student,
the updated student changes the next rollout distribution, and that
distribution determines where the teacher is queried.  A locally useful
correction can therefore alter future state visitation before its downstream
effect is known.  Empirical studies find that OPD depends on compatible
teacher--student reasoning patterns, useful residual teacher capability, and a
small shared set of high-probability tokens
\citep{rethinkingopd}.  Initialization and target construction matter as well:
teacher-compatible warm-up can rescue an otherwise failing run
\citep{simpleopd}, while outcome-calibrated targets can improve the coverage of
plausible continuations \citep{spot}.

Recent methods modify the learning signal by comparing frozen models.
OPD$^2$ transfers the probability change between a post-trained teacher and its
base model \citep{opd2}.  W2S-OPD adds a positive-minus-negative weak-model
contrast to a frozen student-base anchor, yielding a proxy teacher adjacent to
the student \citep{w2sopd}.  These results establish that a model difference
can encode a useful capability direction.  They leave open a different
question: \emph{what happens when the auxiliary model is not a fixed direction
but a second learner on the same on-policy states?}

We study this question through Weak-Driven Learning (WDL), which originally
mixes historical and current logits during supervised fine-tuning
\citep{wdl}.  WDL-OPD transfers the joint-logit principle to on-policy
distillation.  An \emph{anchor} policy generates trajectories; an
\emph{auxiliary} policy sees the same prefixes; their distributions are
geometrically mixed on the anchor's top-$k$ support; and both policies are
updated by the reverse KL from this mixture to a frozen teacher.  At inference
time, either branch can be selected, so the second policy adds training cost but
no mandatory serving cost.

The key distinction is co-adaptation.  If the auxiliary is frozen, exact
mixture matching reduces algebraically to a static extrapolated target.  If it
is trainable, the loss identifies the mixture more strongly than either branch:
the anchor and auxiliary can divide a teacher correction according to their
different Jacobians, optimizer states, and rollout roles.  This suggests an
\emph{optimization-buffer hypothesis}: part of a sharp correction may be
absorbed by the non-rollout branch instead of immediately moving the future
state distribution.  The hypothesis is testable, but it is not implied by
better endpoint accuracy alone.

Our contributions are:
\begin{itemize}
    \item We formulate mixture-constrained two-policy OPD and give an explicit
    algorithm whose rollout, support construction, loss, and gradient paths
    match the implementation.
    \item We characterize its relation to frozen-delta methods.  A
    frozen-auxiliary limit yields an anchor-plus-contrast proxy target, while
    joint training introduces a branch-level null space and makes both policies
    candidate final artifacts.
    \item We consolidate experiments across two model scales and two domains,
    including negative runs and independent re-evaluations.  We state the
    curriculum, initialization, and compute confounds that separate current
    evidence from a controlled causal comparison.
\end{itemize}

\section{Background and Related Work}

\subsection{On-policy distillation}

Let $x$ be a prompt, $y\sim\pi_\theta(\cdot\mid x)$ a student rollout, and
$s_t=(x,y_{<t})$ a visited prefix.  Generalized knowledge distillation trains
on student-generated sequences and permits different divergences between
student and teacher distributions \citep{agarwal2024opd}.  A sampled-token OPD
signal is
\begin{equation}
    A_t^{\mathrm{OPD}}
    =
    \log\pi_T(y_t\mid s_t)-\log\pi_\theta(y_t\mid s_t),
    \label{eq:sampled_opd}
\end{equation}
whereas distributional implementations minimize a forward or reverse KL on
the full vocabulary or a shared truncated support.  We use reverse KL because
it is the loss used by the recorded implementation and recent reasoning OPD
systems \citep{rethinkingopd,w2sopd}.

The success of OPD is conditional.  Compatible reasoning patterns and residual
teacher capability are empirically important, and most shared mass often lies
on a small token set \citep{rethinkingopd}.  EffOPD studies early update
directions and accelerates training by extrapolating along them
\citep{effopd}; Simple-OPD studies teacher-compatible warm-up
\citep{simpleopd}; and SPOT allocates sparse continuation probes before
constructing outcome-calibrated targets \citep{spot}.  These approaches modify
initialization, target quality, or update efficiency.  Our method instead
changes how trainable capacity is coupled to a fixed teacher target.

\subsection{Structured reasoning and adaptive feedback}

A broader survey frames reasoning, interaction, world-model simulation, and
multi-agent coordination as levels of epistemic exploration, characterized by
information gain, value improvement, and epistemic reachability
\citep{ban2026epistemic}.  At the reasoning level, Graph-MCTS uses Monte Carlo
tree search to traverse graph-structured knowledge \citep{liu2025monte}, while
Graph-O1 couples multi-turn graph
interaction with end-to-end reinforcement learning \citep{liu2025graph}.
HyperKGR represents query-specific hierarchical reasoning paths with a
hyperbolic graph neural network \citep{liu2025hyperkgr}, and MixRAG dynamically
routes queries across specialized graph retrievers before filtering the
retrieved evidence \citep{liu2026mixrag}.  These methods structure search,
retrieval, or representation for graph reasoning; WDL-OPD instead studies how
two trainable language-model policies share a teacher correction on the same
generated states.

Other methods adapt the learning signal over interaction or training time.
T-POP learns a lightweight reward function from online pairwise preferences
and uses it to steer a frozen model's decoding \citep{qu2026tpop}.
Workflow-R1 aligns policy updates with semantic Think--Action subsequences in
multi-turn workflow construction \citep{kong2026workflowr1}; SAGE-RL trains
efficient stopping patterns discovered by self-aware sampling
\citep{huang2026does}; history-aware difficulty weighting corrects systematic
bias in group-relative advantages \citep{yang2026your}; and policy-improvement
feedback modulates local updates according to measured progress across policy
iterations \citep{wang2026policy}.  WDL-OPD is complementary to these choices:
it retains dense teacher-distribution supervision but changes the trainable
policy structure through which that supervision is absorbed.

\subsection{Multi-model and multi-agent collaboration}

Multi-model learning systems also differ in what their auxiliary components
exchange and when those components are required.  Transformer Copilot trains a
second model from an evolving mistake log and fuses its logit correction with
the pilot at inference time \citep{zou2025transformer}.  HACRL instead shares
verified rollouts bidirectionally among heterogeneous agents during training
while preserving independent inference \citep{zhang2026heterogeneous}.

At the system level, MASPOB combines bandit search, graph neural networks, and
coordinate ascent to optimize prompts over a fixed multi-agent topology
\citep{hong2026maspob}.  LatentMAS enables training-free communication through
a shared latent working memory \citep{zou2026latent}, whereas RecursiveMAS
connects heterogeneous agents in a recursive latent loop and co-optimizes the
system through shared credit assignment \citep{zou2026recursivemas}.  WDL-OPD
does not optimize prompts or agent messages: its two policies are coupled by a
single geometric-mixture distillation objective on anchor-generated rollouts,
and either learned branch can be deployed alone.

\subsection{Static model differences and weak checkpoints}

OPD$^2$ defines a centered delta signal from a post-trained teacher and its
pre-training base, then keeps sign-consistent updates relative to standard OPD
\citep{opd2}.  W2S-OPD forms
\begin{equation}
    \pi_{\mathrm{proxy}}
    \propto
    \exp\!\left[z_B+\alpha(z^+-z^-)\right],
    \label{eq:w2s}
\end{equation}
where the anchor $B$ and the positive/negative contrast models are frozen
\citep{w2sopd}.  Only one student is optimized in both methods.

WDL takes a different view of weak checkpoints: they are optimization memory,
not inferior teachers.  Its supervised objective mixes historical and current
logits before applying a target, allowing a weak state to compete with and
reshape a stronger state \citep{wdl}.  WDL-OPD retains this joint-logit
operation but replaces fixed supervised examples with anchor-generated states
and a frozen teacher distribution.

\begin{table}[H]
    \centering
    \small
    \resizebox{\linewidth}{!}{
    \begin{tabular}{@{}llllcc@{}}
        \toprule
        Method & Rollout source & Distillation target & Auxiliary status &
        Trainable policies & Moving auxiliary \\
        \midrule
        Standard OPD & Student & Frozen teacher & None & 1 & No \\
        OPD$^2$ & Student & Teacher/base delta & Frozen & 1 & No \\
        W2S-OPD & Student & Base-plus-contrast proxy & Frozen & 1 & No \\
        WDL-OPD & Anchor & Teacher through mixture & Trainable & 2 & Yes \\
        \bottomrule
    \end{tabular}}
    \caption{The methodological axis studied here.  WDL-OPD does not claim
    novelty for model subtraction or logit extrapolation; it makes the
    auxiliary distribution part of the optimization state.}
    \label{tab:method_comparison}
\end{table}

\section{Mixture-Constrained Co-Training}
\label{sec:method}

\subsection{Anchor-policy rollouts and shared support}

We maintain an anchor policy $\pi_A$, an auxiliary policy $\pi_U$, and a frozen
teacher $\pi_T$.  Only the anchor generates trajectories:
\begin{equation}
    y\sim\pi_A(\cdot\mid x).
    \label{eq:rollout}
\end{equation}
The resulting prefixes are on-policy for $\pi_A$ but not for $\pi_U$.  We use
\emph{on-policy} only in this precise sense.

At state $s_t$, let $z_i(\cdot\mid s_t)$ be the logits of model
$i\in\{A,U,T\}$ and let $\mathcal K_t$ contain the $k$ highest-logit tokens of
the anchor.  All three models are evaluated on this same support.  Their
support-conditional log probabilities are
\begin{equation}
    \ell_i(v\mid s_t)
    =
    z_i(v\mid s_t)
    -
    \log\!\sum_{u\in\mathcal K_t}\exp z_i(u\mid s_t),
    \quad v\in\mathcal K_t.
    \label{eq:support_logprob}
\end{equation}
The shared support ensures that mixture and teacher distributions refer to the
same events.  It also follows the empirical observation that a small common
token set contains most of the mass in successful OPD
\citep{rethinkingopd}.

\subsection{Geometric mixture and reverse KL}

For $\lambda\in(0,1)$, define the mixed log score
\begin{equation}
    m_t(v)
    =
    \lambda\ell_A(v\mid s_t)
    +(1-\lambda)\ell_U(v\mid s_t),
    \label{eq:mix_score}
\end{equation}
and renormalize once more:
\begin{equation}
    \ell_M(v\mid s_t)
    =
    m_t(v)-\log\!\sum_{u\in\mathcal K_t}\exp m_t(u).
    \label{eq:mix_logprob}
\end{equation}
Thus $q_M=\exp\ell_M$ is the normalized geometric mean
$q_A^\lambda q_U^{1-\lambda}$.  The token loss is
\begin{equation}
    \mathcal L_t
    =
    D_{\mathrm{KL}}\!\left(q_M(\cdot\mid s_t)
    \,\Vert\,q_T(\cdot\mid s_t)\right)
    =
    \sum_{v\in\mathcal K_t}
    q_M(v\mid s_t)\bigl[\ell_M(v\mid s_t)-\ell_T(v\mid s_t)\bigr].
    \label{eq:token_loss}
\end{equation}
With response mask $r_t$, the batch objective is the token mean
\begin{equation}
    \mathcal L_{\mathrm{WDL\text{-}OPD}}
    =
    \frac{\sum_{x,y,t}r_t\mathcal L_t}
         {\sum_{x,y,t}r_t}.
    \label{eq:wdlopd}
\end{equation}
The teacher and discrete support selection are detached.  Gradients from
Equation~\ref{eq:wdlopd} update \emph{both} trainable policies, each with its
own optimizer state.  The mixture KL is the policy-training loss, not an
auxiliary reward added to a separate policy-gradient objective.

\begin{algorithm}[t]
    \caption{WDL-OPD with anchor-generated trajectories}
    \label{alg:wdlopd}
    \begin{algorithmic}[1]
        \Require prompt set $\mathcal D$; anchor $\pi_A$; auxiliary $\pi_U$;
        frozen teacher $\pi_T$; mixture weight $\lambda$; support size $k$
        \For{each optimization step}
            \State sample a prompt batch $\mathcal B\subset\mathcal D$
            \State generate $y\sim\pi_A(\cdot\mid x)$ for each $x\in\mathcal B$
            \Comment{no gradient through sampling}
            \State initialize $\mathcal L\gets 0$ and valid-token count $N\gets0$
            \For{each visited response state $s_t=(x,y_{<t})$}
                \State $\mathcal K_t\gets\Call{TopK}{z_A(\cdot\mid s_t),k}$
                \State compute $\ell_A,\ell_U,\ell_T$ on $\mathcal K_t$ using
                Equation~\ref{eq:support_logprob}
                \State $m_t\gets\lambda\ell_A+(1-\lambda)\ell_U$
                \State $\ell_M\gets m_t-\Call{LogSumExp}{m_t}$
                \State $\mathcal L\gets\mathcal L+
                \sum_{v\in\mathcal K_t}e^{\ell_M(v)}
                [\ell_M(v)-\ell_T(v)]$
                \State $N\gets N+1$
            \EndFor
            \State backpropagate $\mathcal L/N$ to $\pi_A$ and $\pi_U$
            \State step both optimizers; keep $\pi_T$ fixed
        \EndFor
        \Ensure validation-selected anchor or auxiliary checkpoint
    \end{algorithmic}
\end{algorithm}

\subsection{Co-adaptive degrees of freedom}

Let
\begin{equation}
    g_t=\nabla_{m_t}D_{\mathrm{KL}}(q_M\Vert q_T).
\end{equation}
Ignoring the discrete top-$k$ operation, the support-logit derivatives are
\begin{equation}
    \nabla_{z_A}\mathcal L_t=\lambda g_t,
    \qquad
    \nabla_{z_U}\mathcal L_t=(1-\lambda)g_t.
    \label{eq:branchgrads}
\end{equation}
They point in the same mixed-logit direction but are transformed by different
parameter Jacobians and optimizer histories.  Furthermore, perturbations
satisfying
\begin{equation}
    \lambda\Delta z_A+(1-\lambda)\Delta z_U
    =c(s_t)\mathbf 1
    \label{eq:null}
\end{equation}
leave $q_M$ unchanged.  The objective therefore identifies the mixture more
strongly than either branch.  Equation~\ref{eq:null} is not by itself a claim
that the optimizer will discover a beneficial decomposition; it identifies
the freedom that makes such a decomposition possible.

\subsection{Frozen-auxiliary limit}

\begin{proposition}[Equivalence to an extrapolated target]
\label{prop:frozen}
Fix the auxiliary support logits $z_U$.  If the mixture can exactly match the
teacher on $\mathcal K_t$, then every exact anchor solution satisfies
\begin{equation}
    z_A^*
    =
    z_T+\frac{1-\lambda}{\lambda}(z_T-z_U)
    +c'(s_t)\mathbf 1.
    \label{eq:proxyrelation}
\end{equation}
\end{proposition}

\begin{proof}
Equality of two softmax distributions implies equality of their logits up to a
state-dependent scalar.  Hence
$\lambda z_A+(1-\lambda)z_U=z_T+c(s_t)\mathbf 1$.  Solving for $z_A$ and
absorbing $c/\lambda$ into $c'$ gives
Equation~\ref{eq:proxyrelation}.
\end{proof}

The right-hand side is a teacher anchor plus a teacher-minus-auxiliary
contrast.  If $U$ is the teacher's base checkpoint, the contrast is the raw
direction used by OPD$^2$; as a logit proxy it has the same
anchor-plus-contrast form as W2S-OPD.  The novelty of WDL-OPD is therefore not
subtraction or extrapolation.  It is the transition from a fixed contrast to a
moving auxiliary that shares the loss with the rollout policy.

\section{Experiments}
\label{sec:experiments}

\subsection{Research questions}

The experiments ask three questions:
\begin{enumerate}
    \item \textbf{RQ1:} Does the two-policy system produce a stronger student
    checkpoint than the recorded single-policy alternatives?
    \item \textbf{RQ2:} Do anchor and auxiliary branches behave as redundant
    copies, or can their trajectories and final capabilities diverge?
    \item \textbf{RQ3:} In code settings where single-policy OPD degrades, does
    co-training avoid the same observable failure signature?
\end{enumerate}
The current evidence answers these questions descriptively.  It is not a
single factorial study, and we keep that distinction explicit throughout.

\subsection{Models, data, and evaluation}

We study Qwen3-1.7B and Qwen3-4B models \citep{qwen3}.  The mathematical
reasoning line uses execution- or answer-verified prompts.  The 4B
single-policy and two-policy stages both use a 104,935-prompt pool.  The 1.7B
single-policy stage uses a 5,662-prompt \emph{sweet-zone} curriculum, after
which the two-policy stage switches to the full 104,935-prompt pool.

Math evaluation includes GSM8K \citep{cobbe2021gsm8k}, MATH500
\citep{hendrycks2021math,lightman2024verify}, AMC and AIME competition sets,
OlympiadBench \citep{he2024olympiad}, and Minerva-style quantitative reasoning
\citep{lewkowycz2022minerva}.  Unless stated otherwise, we report sampled
accuracy with four responses per problem and eight for AIME.

The code line uses 20,247 execution-verified training problems and a disjoint
500-problem development set.  External evaluation uses HumanEval+ and MBPP+
from EvalPlus \citep{liu2023evalplus}, BigCodeBench
\citep{zhuo2025bigcodebench}, and LiveCodeBench
\citep{jain2024livecodebench}.  These external scores use the same chat-format
harness within each scale; development scores are independently re-evaluated
from preserved checkpoints.

\subsection{Optimization and comparison scope}

All headline WDL-OPD runs use $\lambda=0.5$, $k=16$, temperature $1.0$, and
four anchor rollouts per prompt.  The 4B math run uses a constant learning rate
of $10^{-6}$ for 300 steps and a 4096-token response cap.  Code runs use
$10^{-6}$ for 300 steps, with response caps between 4096 and 6144 tokens.  The
late 1.7B math continuation reduces the rate to $5\times10^{-7}$.  Both branch
optimizers are stepped once per batch.  Appendix~\ref{app:details} records the
run-level provenance.

\begin{table}[t]
    \centering
    \small
    \resizebox{\linewidth}{!}{
    \begin{tabular}{@{}lllll@{}}
        \toprule
        Setting & Anchor initialization & Auxiliary initialization &
        Frozen teacher & Main evaluation \\
        \midrule
        Math 4B & Single-OPD checkpoint & Qwen3-4B-Base &
        4B math teacher & MATH500 avg@4 \\
        Math 1.7B & Trained 1.7B checkpoint & Qwen3-1.7B-Base &
        1.7B math teacher & Seven math sets \\
        Code 4B & Trained code checkpoint & Qwen3-4B-Base &
        Qwen3-4B-Instruct-2507 & Dev500 and code suite \\
        Code 1.7B & GRPO code checkpoint & Qwen3-1.7B-Base &
        Qwen3-4B-Instruct-2507 & Dev500 and code suite \\
        \bottomrule
    \end{tabular}}
    \caption{Initialization and teacher map.  The unequal initializations are
    part of the current method recipe, but they are also a confound for causal
    comparison with single-policy OPD.}
    \label{tab:initializations}
\end{table}

Table~\ref{tab:initializations} shows why the results should not be read as a
compute-matched benchmark against OPD$^2$ or W2S-OPD.  The code experiments
also include a health rule: a checkpoint is reportable only if the development
score is confirmed without severe repetition or length pathology.  For the 4B
run, this rule selects step 90 before repetition appears at steps 100--109.

\section{Results}

\subsection{RQ1: performance across scale and domain}

\begin{table}[t]
    \centering
    \small
    \resizebox{\linewidth}{!}{
    \begin{tabular}{@{}llrrrrr@{}}
        \toprule
        Domain & Scale & Base & Single OPD & WDL-OPD anchor & WDL-OPD aux. & Teacher \\
        \midrule
        Math & 4B   & 0.544 & 0.630 & 0.640 & \textbf{0.685} & 0.805 \\
        Math & 1.7B & 0.390 & 0.521 & \textbf{0.585} & 0.518 & 0.799 \\
        Code & 4B   & 0.463 & unstable (6/6) & \textbf{0.637} & 0.472 & 0.784 \\
        Code & 1.7B & 0.184 & 0.106 & \textbf{0.375} & 0.283 & 0.784 \\
        \bottomrule
    \end{tabular}}
    \caption{Project-level evidence map.  Math reports sampled MATH500
    accuracy; code reports independent Dev500 re-evaluation.  Bold marks the
    strongest student branch, not the teacher.  The unstable marker means that
    all six recorded 4B single-policy configurations degraded or developed
    entropy/repetition pathologies.  Protocols are not matched across every
    row.}
    \label{tab:fourcells}
\end{table}

Table~\ref{tab:fourcells} summarizes the four scale--domain cells.  In math,
the best two-policy branch improves over the selected single-OPD checkpoint by
$5.5$ points at 4B and $6.4$ points at 1.7B.  In code, the contrast is larger:
the 4B single-policy sweep does not yield a healthy final candidate, and the
1.7B single-policy run falls below base.  WDL-OPD yields the strongest recorded
student checkpoint in all four cells.

This result answers RQ1 for the project record, not for a controlled method
comparison.  In particular, the 1.7B math curriculum changes between stages
and the code anchor is already trained before the two-policy stage.  The table
shows that the recipe can produce a stronger artifact; it does not isolate
which component causes the gain.

\subsection{RQ2: branch dynamics and specialization}

The 4B math trajectory is informative because both branches are evaluated
during training.  The anchor begins at the selected single-OPD checkpoint
($0.630$), whereas the auxiliary begins at base ($0.544$).

\begin{table}[t]
    \centering
    \small
    \begin{tabular}{@{}lrrrrrrr@{}}
        \toprule
        Step & 0 & 40 & 80 & 150 & 190 & 250 & 290 \\
        \midrule
        Anchor & 0.630 & 0.628 & 0.609 & -- & -- & -- & 0.640 \\
        Auxiliary & 0.544 & 0.636 & 0.666 & 0.678 & 0.683 & \textbf{0.685} & 0.683 \\
        \bottomrule
    \end{tabular}
    \caption{Qwen3-4B MATH500 avg@4 during co-training.  Dashes denote steps at
    which the anchor was not retained in the evaluation record.}
    \label{tab:trajectory}
\end{table}

The initially weaker auxiliary crosses the anchor by step 40 and becomes the
best final artifact.  The anchor is non-monotonic, falling to $0.609$ before
recovering to $0.640$.  This trajectory is inconsistent with a picture in
which both policies simply move together toward one midpoint.  It is
consistent with branch-level freedom, although endpoint scores alone cannot
identify the underlying gradient decomposition.

\begin{table}[t]
    \centering
    \small
    \resizebox{\linewidth}{!}{
    \begin{tabular}{@{}lrrrrrrr@{}}
        \toprule
        Model & GSM8K & MATH500 & AMC23 & Olympiad & Minerva & AIME24 & AIME25 \\
        \midrule
        Base & 0.481 & 0.416 & 0.178 & 0.162 & 0.118 & 0.029 & 0.013 \\
        Single OPD & 0.621 & 0.509 & 0.266 & 0.217 & 0.174 & 0.033 & 0.025 \\
        WDL-OPD aux. & 0.627 & 0.516 & 0.259 & 0.192 & 0.138 & 0.025 & 0.017 \\
        WDL-OPD anchor & \textbf{0.712} & \textbf{0.585} & \textbf{0.316} &
        \textbf{0.246} & \textbf{0.183} & \textbf{0.062} & \textbf{0.054} \\
        \bottomrule
    \end{tabular}}
    \caption{Qwen3-1.7B mathematical reasoning.  Accuracy is avg@$n$ with a
    4k response budget; AIME uses eight samples and the other tasks use four.
    The single- and two-policy stages use different training pools.}
    \label{tab:math7}
\end{table}

At 1.7B, the roles reverse: the anchor is the stronger final branch.
Table~\ref{tab:math7} shows that it improves over single OPD on all seven
benchmarks.  The gain is largest on GSM8K ($+9.1$ points) and MATH500
($+7.6$), and smaller on Minerva and AIME.  Across math and code, the best
deployment branch is therefore not fixed in advance.  This answers RQ2:
branch identity matters empirically, and validation must retain both branches.

\subsection{RQ3: code stability and external evaluation}

The code line exposes a repeatable failure signature.  Six 4B single-policy
variants span two prompt renderings, two teachers, and multiple learning rates;
all plateau or decline while entropy or repetition increases.  A separate
1.7B single-policy run falls from $0.181$ to $0.106$ as response entropy grows
from approximately $2.4$ to $4.2$.  Representative 4B runs increase from
approximately $1.4$ to $2.9$.  These seven runs motivate stabilization, but
they do not prove that entropy growth is the root cause.

The two-policy runs do not remove every pathology.  The selected 4B checkpoint
reaches an online development peak of $0.669$ at step 90, then develops
repetition around steps 100--109.  Independent re-evaluation of the preserved
healthy checkpoint gives $0.637$, which is the value reported in
Table~\ref{tab:fourcells}.  The 1.7B run is non-monotonic but recovers to an
online $0.408$ at its terminal checkpoint; independent re-evaluation gives
$0.375$.

\begin{table}[t]
    \centering
    \small
    \resizebox{\linewidth}{!}{
    \begin{tabular}{@{}llrrrr@{}}
        \toprule
        Scale & Model & HumanEval+ & MBPP+ & BigCodeBench & LiveCodeBench \\
        \midrule
        4B & Base & 0.707 & \textbf{0.664} & 0.469 & 0.018 \\
        4B & WDL-OPD anchor & 0.774 & 0.643 & 0.496 & 0.253 \\
        4B & Teacher & \textbf{0.854} & 0.653 & \textbf{0.511} & \textbf{0.365} \\
        \midrule
        1.7B & Base & 0.598 & \textbf{0.595} & 0.319 & 0.004 \\
        1.7B & WDL-OPD auxiliary & 0.634 & 0.574 & \textbf{0.339} & 0.038 \\
        1.7B & WDL-OPD anchor & \textbf{0.640} & 0.585 & 0.328 & \textbf{0.062} \\
        \bottomrule
    \end{tabular}}
    \caption{External code evaluation under a common chat-format harness within
    each scale.  Bold marks the best deployable student result in each block;
    the 4B teacher is a reference upper bound.}
    \label{tab:code}
\end{table}

Table~\ref{tab:code} shows that the gains are not uniform.  At 4B,
LiveCodeBench increases from $0.018$ at base to $0.253$, while MBPP+ decreases.
At 1.7B, the anchor leads on HumanEval+ and LiveCodeBench, whereas the
auxiliary leads on BigCodeBench.  Co-training therefore avoids the exact
single-policy failure signature long enough to produce useful checkpoints and
supports RQ3 descriptively.  The MBPP+ regression and later 4B repetition rule
out a blanket stability or monotonic-improvement claim.

\section{Analysis and Discussion}

\paragraph{What is new relative to OPD$^2$ and W2S-OPD?}
All three methods exploit differences between model distributions.
Proposition~\ref{prop:frozen} shows that our frozen-auxiliary limit is already
a proxy-target construction.  We therefore claim neither teacher-minus-base
deltas nor logit extrapolation as novel.  The new object is a moving auxiliary
policy that shares the mixture loss with the rollout anchor.  This creates
branch-level degrees of freedom, makes either branch a possible final artifact,
and permits specialization that a fixed proxy cannot express.

\paragraph{The optimization-buffer hypothesis.}
Equation~\ref{eq:branchgrads} splits the immediate mixed-logit derivative while
the anchor alone controls future state visitation.  One possible mechanism is
that the auxiliary absorbs part of a sharp correction, reducing how quickly
the rollout distribution moves.  The observed entropy failures and branch
trajectories are compatible with this explanation, but several alternatives
remain: extra trainable capacity, different initializations, implicit
ensembling in the loss, or checkpoint selection may account for part of the
gain.

\paragraph{Measurements that can distinguish mechanisms.}
A direct test should log per-branch gradient norms, cosine similarity,
parameter displacement, entropy, top-$k$ support mass, and state-distribution
drift.  A dual-compute control should train two independent policies toward the
teacher without a shared mixture.  If extra capacity alone explains the
result, that control should match WDL-OPD.  If mixture coupling matters, its
branch trajectories and rollout drift should differ even at matched FLOPs.

\paragraph{Controlled baseline matrix.}
The definitive comparison should hold prompts, teacher, initialization,
rollout states, support, optimizer, and wall-clock budget fixed while comparing
standard OPD, OPD$^2$, W2S-style fixed proxy distillation, a frozen-auxiliary
mixture, fully trainable WDL-OPD, and the independent two-policy control.
Simple-OPD-style warm-up \citep{simpleopd} should either be shared by every row
or treated as a separate factor.  This matrix is more informative than adding
another unmatched endpoint to Table~\ref{tab:fourcells}.

\section{Limitations}

The current study has five material limitations.  First, the 1.7B math
curriculum differs between single- and two-policy stages.  Second, several code
comparisons use different starting checkpoints, and the selected 4B checkpoint
precedes a later repetition failure.  Third, training two policies costs more
than standard OPD; a FLOP- and wall-clock-matched control is missing.  Fourth,
we have not yet run OPD$^2$, W2S-OPD, Simple-OPD, or SPOT in the same harness.
Fifth, all experiments use the Qwen3 family and mostly single-answer reasoning
or executable code.  The conclusions may not transfer to open-ended
generation, different model families, or interactive agents.  These
limitations prevent a universal causal claim that WDL-OPD stabilizes every OPD
setting.

\section{Conclusion}

WDL-OPD extends weak-driven joint-logit training to anchor-generated
trajectories.  It co-trains an anchor and an auxiliary through the reverse KL
between their geometric mixture and a frozen teacher.  The frozen-auxiliary
limit recovers a familiar proxy-target form; the trainable case adds
branch-level freedom and produces distinct candidate models.  Across the
current 1.7B/4B math and code record, the method yields stronger student
checkpoints and avoids several single-policy failure trajectories.  The
evidence is promising but not yet compute- or curriculum-matched.  The next
scientific step is therefore not a broader claim, but the controlled matrix
that distinguishes mixture coupling from initialization, capacity, and target
design.

\bibliographystyle{iclr2027_conference}
\bibliography{references}

@inproceedings{agarwal2024opd,
  title={On-Policy Distillation of Language Models: Learning from Self-Generated Mistakes},
  author={Agarwal, Rishabh and Vieillard, Nino and Zhou, Yongchao and Stanczyk, Piotr and Ramos, Sabela and Geist, Matthieu and Bachem, Olivier},
  booktitle={International Conference on Learning Representations},
  year={2024},
  eprint={2306.13649},
  archivePrefix={arXiv}
}

@misc{thinkingopd,
  title={On-Policy Distillation},
  author={Lu, Kevin and Thinking Machines Lab},
  howpublished={Thinking Machines Lab blog},
  year={2025},
  note={Published October 27, 2025}
}

@article{rethinkingopd,
  title={Rethinking On-Policy Distillation of Large Language Models: Phenomenology, Mechanism, and Recipe},
  author={Li, Yaxuan and Zuo, Yuxin and He, Bingxiang and Zhang, Jinqian and Xiao, Chaojun and Qian, Cheng and Yu, Tianyu and Gao, Huan-ang and Yang, Wenkai and Liu, Zhiyuan and Ding, Ning},
  journal={arXiv preprint arXiv:2604.13016},
  year={2026}
}

@article{effopd,
  title={Learning to Foresee: Unveiling the Unlocking Efficiency of On-Policy Distillation},
  author={Cai, Yuchen and Cao, Ding and Lin, Liang and Luo, Chunxi and Xu, Xin and Yang, Kai and Liu, Weijie and Yang, Saiyong and Zhao, Tianxiang and Sun, Guangzhong and Liu, Guiquan and Fang, Junfeng},
  journal={arXiv preprint arXiv:2605.11739},
  year={2026}
}

@article{opd2,
  title={On-Policy Delta Distillation},
  author={Heo, Byeongho and Hwang, Jaehui and Yun, Sangdoo and Han, Dongyoon},
  journal={arXiv preprint arXiv:2607.15161},
  year={2026}
}

@article{w2sopd,
  title={Weak-to-Strong On-Policy Distillation},
  author={Yu, Fangxu and Xu, Weijia and Xu, Michael and Zhou, Tianyi and Lin, Zinan},
  journal={arXiv preprint arXiv:2607.26246},
  year={2026}
}

@article{simpleopd,
  title={{Simple-OPD}: Demystifying Warm-up for On-policy Distillation},
  author={Liu, Tao and Wu, Taiqiang and Zheng, Mao and Luo, Xuan and Yang, Runming and Yang, Xuewei and Wang, Junjie and Yang, Yujiu},
  journal={arXiv preprint arXiv:2608.06802},
  year={2026}
}

@article{spot,
  title={{SPOT}: Sparse Probing and Outcome Calibration for On-Policy Distillation},
  author={Qu, Zikun and Zhang, Min and Kong, Mingze and Shang, Zhiwei and Ban, Yikun and Qiu, Shuang and Dai, Zhongxiang},
  journal={arXiv preprint arXiv:2608.04419},
  year={2026}
}

@article{wdl,
  title={Weak-Driven Learning: How Weak Agents Make Strong Agents Stronger},
  author={Chen, Zehao and Li, Gongxun and Ai, Tianxiang and Huang, Zixuan and Liu, Xiaodong and Li, Yifei and Zhou, Wang and Zhuang, Fuzhen and Liu, Xianglong and Li, Jianxin and Wang, Deqing and Ban, Yikun},
  journal={arXiv preprint arXiv:2602.08222},
  year={2026}
}

@article{qwen3,
  title={{Qwen3} Technical Report},
  author={Yang, An and Li, Anfeng and Yang, Baosong and others},
  journal={arXiv preprint arXiv:2505.09388},
  year={2025}
}

@article{cobbe2021gsm8k,
  title={Training Verifiers to Solve Math Word Problems},
  author={Cobbe, Karl and Kosaraju, Vineet and Bavarian, Mohammad and Chen, Mark and Jun, Heewoo and Kaiser, Lukasz and Plappert, Matthias and Tworek, Jerry and Hilton, Jacob and Nakano, Reiichiro and Hesse, Christopher and Schulman, John},
  journal={arXiv preprint arXiv:2110.14168},
  year={2021}
}

@inproceedings{hendrycks2021math,
  title={Measuring Mathematical Problem Solving With the {MATH} Dataset},
  author={Hendrycks, Dan and Burns, Collin and Kadavath, Saurav and Arora, Akul and Basart, Steven and Tang, Eric and Song, Dawn and Steinhardt, Jacob},
  booktitle={Advances in Neural Information Processing Systems},
  year={2021}
}

@inproceedings{lightman2024verify,
  title={Let's Verify Step by Step},
  author={Lightman, Hunter and Kosaraju, Vineet and Burda, Yura and Edwards, Harri and Baker, Bowen and Lee, Teddy and Leike, Jan and Schulman, John and Sutskever, Ilya and Cobbe, Karl},
  booktitle={International Conference on Learning Representations},
  year={2024},
  eprint={2305.20050},
  archivePrefix={arXiv}
}

@inproceedings{he2024olympiad,
  title={{OlympiadBench}: A Challenging Benchmark for Promoting {AGI} with Olympiad-Level Bilingual Multimodal Scientific Problems},
  author={He, Chaoqun and Luo, Renjie and Bai, Yuzhuo and Hu, Shengding and Thai, Zhen Leng and Shen, Junhao and Hu, Jinyi and Han, Xu and Huang, Yujie and Zhang, Yuxiang and Liu, Jie and Qi, Lei and Liu, Zhiyuan and Sun, Maosong},
  booktitle={Annual Meeting of the Association for Computational Linguistics},
  year={2024}
}

@inproceedings{lewkowycz2022minerva,
  title={Solving Quantitative Reasoning Problems with Language Models},
  author={Lewkowycz, Aitor and Andreassen, Anders and Dohan, David and Dyer, Ethan and Michalewski, Henryk and Ramasesh, Vinay and Slone, Ambrose and Anil, Cem and Schlag, Imanol and Gutman-Solo, Theo and Wu, Yuhuai and Neyshabur, Behnam and Gur-Ari, Guy and Misra, Vedant},
  booktitle={Advances in Neural Information Processing Systems},
  year={2022}
}

@article{liu2023evalplus,
  title={Is Your Code Generated by {ChatGPT} Really Correct? Rigorous Evaluation of Large Language Models for Code Generation},
  author={Liu, Jiawei and Xia, Chunqiu Steven and Wang, Yuyao and Zhang, Lingming},
  journal={arXiv preprint arXiv:2305.01210},
  year={2023}
}

@inproceedings{zhuo2025bigcodebench,
  title={{BigCodeBench}: Benchmarking Code Generation with Diverse Function Calls and Complex Instructions},
  author={Zhuo, Terry Yue and Vu, Minh Chien and Chim, Jenny and others},
  booktitle={International Conference on Learning Representations},
  year={2025},
  eprint={2406.15877},
  archivePrefix={arXiv}
}

@article{jain2024livecodebench,
  title={{LiveCodeBench}: Holistic and Contamination Free Evaluation of Large Language Models for Code},
  author={Jain, Naman and Han, King and Gu, Alex and Li, Wen-Ding and Yan, Fanjia and Zhang, Tianjun and Wang, Sida and Solar-Lezama, Armando and Sen, Koushik and Stoica, Ion},
  journal={arXiv preprint arXiv:2403.07974},
  year={2024}
}

@inproceedings{liu2025monte,
  title={Monte Carlo Tree Search for Graph Reasoning in Large Language Model Agents},
  author={Liu, Lihui},
  booktitle={Proceedings of the 34th ACM International Conference on Information and Knowledge Management},
  pages={4966--4970},
  year={2025},
  doi={10.1145/3746252.3760854}
}

@inproceedings{liu2025hyperkgr,
  title={{HyperKGR}: Knowledge Graph Reasoning in Hyperbolic Space with Graph Neural Network Encoding Symbolic Path},
  author={Liu, Lihui},
  booktitle={Proceedings of the 2025 Conference on Empirical Methods in Natural Language Processing},
  pages={25177--25188},
  publisher={Association for Computational Linguistics},
  address={Suzhou, China},
  year={2025},
  doi={10.18653/v1/2025.emnlp-main.1279},
  url={https://aclanthology.org/2025.emnlp-main.1279/}
}

@inproceedings{liu2026mixrag,
  title={{MixRAG}: Mixture-of-Experts Retrieval-Augmented Generation for Textual Graph Understanding and Question Answering},
  author={Liu, Lihui and Ding, Jiayuan and Mukherjee, Subhabrata and Yang, Carl},
  booktitle={Proceedings of the ACM Web Conference 2026},
  pages={4350--4359},
  year={2026}
}

@article{liu2025graph,
  title={{Graph-O1}: Monte Carlo Tree Search with Reinforcement Learning for Text-Attributed Graph Reasoning},
  author={Liu, Lihui},
  journal={arXiv preprint arXiv:2512.17912},
  year={2025},
  eprint={2512.17912},
  archivePrefix={arXiv},
  primaryClass={cs.CL}
}

@inproceedings{qu2026tpop,
  title={{T-POP}: Test-Time Personalization with Online Preference Feedback},
  author={Qu, Zikun and Zhang, Min and Kong, Mingze and Li, Xiang and Shang, Zhiwei and Wang, Zhiyong and Ban, Yikun and Qiu, Shuang and Shu, Yao and Dai, Zhongxiang},
  booktitle={Proceedings of the 43rd International Conference on Machine Learning},
  series={Proceedings of Machine Learning Research},
  volume={306},
  publisher={PMLR},
  address={Seoul, South Korea},
  year={2026},
  month=jul
}

@inproceedings{hong2026maspob,
  title={{MASPOB}: Bandit-Based Prompt Optimization for Multi-Agent Systems with Graph Neural Networks},
  author={Hong, Zhi and Zhang, Qian and Sun, Jiahang and Shang, Zhiwei and Kong, Mingze and Wang, Xiangyi and Shu, Yao and Dai, Zhongxiang},
  booktitle={Proceedings of the 43rd International Conference on Machine Learning},
  series={Proceedings of Machine Learning Research},
  volume={306},
  publisher={PMLR},
  address={Seoul, South Korea},
  year={2026},
  month=jul
}

@article{kong2026workflowr1,
  title={{Workflow-R1}: Group Sub-sequence Policy Optimization for Multi-turn Workflow Construction},
  author={Kong, Mingze and Qu, Zikun and Zhou, Zhongquan and Liang, Pengyu and Li, Xiang and Shang, Zhiwei and Hong, Zhi and Huang, Kaiyu and Wang, Zhiyong and Dai, Zhongxiang},
  journal={arXiv preprint arXiv:2602.01202},
  year={2026},
  eprint={2602.01202},
  archivePrefix={arXiv},
  primaryClass={cs.AI}
}

@article{yang2026your,
  title={Your Group-Relative Advantage Is Biased},
  author={Yang, Fengkai and Chen, Zherui and Wang, Xiaohan and Lu, Xiaodong and Chai, Jiajun and Yin, Guojun and Lin, Wei and Ma, Shuai and Zhuang, Fuzhen and Wang, Deqing and others},
  journal={arXiv preprint arXiv:2601.08521},
  year={2026},
  eprint={2601.08521},
  archivePrefix={arXiv},
  primaryClass={cs.LG}
}

@article{huang2026does,
  title={Does Your Reasoning Model Implicitly Know When to Stop Thinking?},
  author={Huang, Zixuan and Xia, Xin and Ren, Yuxi and Zheng, Jianbin and Wang, Xuanda and Zhang, Zhixia and Xie, Hongyan and Liang, Songshi and Chen, Zehao and Xiao, Xuefeng and others},
  journal={arXiv preprint arXiv:2602.08354},
  year={2026},
  eprint={2602.08354},
  archivePrefix={arXiv},
  primaryClass={cs.AI}
}

@article{zhang2026heterogeneous,
  title={Heterogeneous Agent Collaborative Reinforcement Learning},
  author={Zhang, Zhixia and Huang, Zixuan and Li, Gongxun and Wang, Huaiyang and Yuan, Chengyi and Xia, Xin and Wang, Deqing and Zhuang, Fuzhen and Ma, Shuai and Ding, Ning and others},
  journal={arXiv preprint arXiv:2603.02604},
  year={2026},
  eprint={2603.02604},
  archivePrefix={arXiv},
  primaryClass={cs.LG}
}

@article{wang2026policy,
  title={Policy Improvement Reinforcement Learning},
  author={Wang, Huaiyang and Li, Xiaojie and Wang, Xiaohan and Zhang, Zhixia and Lu, Xiaodong and Huang, Zixuan and Chai, Jiajun and Yin, Guojun and Wang, Deqing and Zhou, Haoyi and others},
  journal={arXiv preprint arXiv:2604.00860},
  year={2026},
  eprint={2604.00860},
  archivePrefix={arXiv},
  primaryClass={cs.LG}
}

@misc{zou2026recursivemas,
  title={Recursive Multi-Agent Systems},
  author={Zou, Jiaru and Pan, Rui and Qiu, Ruizhong and Lu, Pan and Diao, Shizhe and Jiang, Jindong and Tong, Hanghang and Zhang, Tong and Buehler, Markus J. and He, Jingrui and Zou, James},
  year={2026},
  eprint={2604.25917},
  archivePrefix={arXiv},
  primaryClass={cs.AI},
  url={https://arxiv.org/abs/2604.25917}
}

@inproceedings{zou2026latent,
  title={Latent Collaboration in Multi-Agent Systems},
  author={Zou, Jiaru and Qiu, Ruizhong and Li, Gaotang and Yang, Xiyuan and Tieu, Katherine and Lu, Pan and Shen, Ke and Tong, Hanghang and Choi, Yejin and He, Jingrui and Zou, James and Wang, Mengdi and Yang, Ling},
  booktitle={Proceedings of the 43rd International Conference on Machine Learning},
  series={Proceedings of Machine Learning Research},
  volume={306},
  publisher={PMLR},
  address={Seoul, South Korea},
  year={2026},
  month=jul,
  url={https://openreview.net/forum?id=syG9I9ofd8}
}

@inproceedings{zou2025transformer,
  title={Transformer Copilot: Learning from the Mistake Log in {LLM} Fine-Tuning},
  author={Zou, Jiaru and Ban, Yikun and Li, Zihao and Qi, Yunzhe and Qiu, Ruizhong and Yang, Ling and He, Jingrui},
  booktitle={Advances in Neural Information Processing Systems},
  volume={38},
  year={2025},
  url={https://openreview.net/forum?id=MRvxlTlkNQ}
}

@article{ban2026epistemic,
  title={Epistemic Exploration Toward Artificial General Intelligence},
  author={Ban, Yikun and Yang, Fengkai and Chen, Fangzheng and Wang, Yibo and Chen, Zhijun and Li, Zhongyi and Huang, Zixuan and Zhang, Xiaoyuan and Li, Gongxun and Chen, Zehao and others},
  journal={Zenodo},
  year={2026},
  doi={10.5281/zenodo.20201415},
  url={https://doi.org/10.5281/zenodo.20201415}
}

\appendix

\section{Reproducibility Details}
\label{app:details}

\subsection{Recorded run configurations}

\begin{table}[h]
    \centering
    \small
    \resizebox{\linewidth}{!}{
    \begin{tabular}{@{}lrrrrrl@{}}
        \toprule
        Setting & Train prompts & Steps & Learning rate & Rollouts & Response cap &
        Notes \\
        \midrule
        Math 4B & 104,935 & 300 & $10^{-6}$ & 4 & 4096 &
        anchor=single OPD, auxiliary=base \\
        Math 1.7B & 104,935 & 800 & $10^{-6}\!\rightarrow5{\times}10^{-7}$ &
        4 & 4096 & single-OPD precursor used 5,662 prompts \\
        Code 4B & 20,247 & 300 & $10^{-6}$ & 4 & 4096--6144 &
        reportable checkpoint selected by health rule \\
        Code 1.7B & 20,247 & 300 & $10^{-6}$ & 4 & 6144 &
        anchor initialized from GRPO step 120 \\
        \bottomrule
    \end{tabular}}
    \caption{Run-level parameters recovered from launch scripts and experiment
    records.  All rows use $\lambda=0.5$, $k=16$, teacher temperature $1.0$,
    and token-mean aggregation.}
    \label{tab:run_details}
\end{table}

The implementation uses separate optimizer states for anchor and auxiliary,
bfloat16 model execution, dynamic micro-batching, and the anchor-only top-$k$
strategy.  The same learning rate drives both branches within a run.  The
teacher distribution is detached before loss construction.  When WDL-OPD is
enabled, the implementation fails explicitly if the auxiliary module, top-$k$
targets, or teacher-on-anchor support probabilities are missing; it does not
silently fall back to standard OPD.

\subsection{Evaluation and checkpoint selection}

Math results use temperature $0.7$ with a 4096-token generation budget.
GSM8K, MATH500, OlympiadBench, and Minerva use four samples; AMC and AIME use
eight where listed in Table~\ref{tab:math7}.  The 4B trajectory table reports
the recorded avg@4 curve and does not interpolate missing anchor evaluations.

Code development evaluation samples two responses per problem in the
thinking-enabled chat format.  External code evaluation uses the same
chat-format harness within a scale.  The 4B code checkpoint at step 90 was
preserved after reaching the best healthy validation score; later checkpoints
were excluded after the predeclared repetition checks fired.  The reported
$0.637$ is an independent rerun of that preserved checkpoint rather than its
online $0.669$ peak.  The 1.7B value $0.375$ is likewise an independent rerun
of the terminal checkpoint.

\subsection{Minimal implementation identity}

For every response token, the implementation computes
\begin{align}
    \mathrm{mix}
    &=\lambda\,\log p_A+(1-\lambda)\,\log p_U,\\
    \log q_M
    &=\mathrm{mix}-\operatorname{logsumexp}(\mathrm{mix}),\\
    \log q_T
    &=\log p_T-\operatorname{logsumexp}(\log p_T),\\
    \mathcal L_t
    &=\sum_v q_M(v)\,[\log q_M(v)-\log q_T(v)].
\end{align}
Only $p_T$ is detached.  Consequently, this is exactly the geometric-mixture
reverse KL in Algorithm~\ref{alg:wdlopd}, not a sampled-token reward or a
teacher-only surrogate.

\section{Required Controlled Ablations}

For clarity, the highest-priority missing experiment is a single matrix with
fixed prompts, rollout trajectories, teacher, initialization, token support,
optimizer, update count, and total compute.  The rows should be:
\begin{enumerate}
    \item standard single-policy OPD;
    \item OPD$^2$ with the teacher's base model;
    \item W2S-style frozen proxy distillation;
    \item frozen-auxiliary geometric-mixture training;
    \item fully trainable WDL-OPD; and
    \item two independently trained policies without mixture coupling.
\end{enumerate}
Reporting anchor and auxiliary curves, not only the best selected branch, is
necessary to test the co-adaptation explanation.

\end{document}